\documentclass[letterpaper]{article} 
\usepackage{aaai24}  
\usepackage{times}  
\usepackage{helvet}  
\usepackage{courier}  
\usepackage[hyphens]{url}  
\usepackage{graphicx} 
\usepackage{natbib}  
\usepackage{caption} 
\usepackage{algorithm}
\usepackage{algpseudocode}
\usepackage{booktabs}

\usepackage{xspace}
\usepackage{array}
\usepackage{adjustbox}
\usepackage{bm}
\usepackage{amsmath,amsthm,amsfonts}
\usepackage{amssymb}
\usepackage{xcolor}
\newtheorem{thm}{Theorem}

\DeclareMathOperator{\diag}{diag}
\DeclareMathOperator{\reals}{\mathbb{R}}

\DeclareMathOperator*{\argmax}{arg\,max}

\title{A Surprisingly Simple Continuous-Action POMDP Solver: Lazy Cross-Entropy Search Over Policy Trees}
\author {
    Marcus Hoerger\textsuperscript{\rm 1},
    Hanna Kurniawati\textsuperscript{\rm 2},
    Dirk Kroese\textsuperscript{\rm 1},
    Nan Ye\textsuperscript{\rm 1}
}
\affiliations {
    \textsuperscript{\rm 1}School of Mathematics \& Physics, The University of Queensland, Queensland, Australia\\
    \textsuperscript{\rm 2}School of Computing, Australian National University, ACT, Australia\\
    m.hoerger@uq.edu.au, hanna.kurniawati@anu.edu.au, kroese@maths.uq.edu.au, nan.ye@uq.edu.au
}

\usepackage{cleveref}
\crefname{equation}{Equation}{Equations}
\crefname{thm}{Theorem}{Theorems}

\begin{document}

\newcommand{\belSpace}{\ensuremath{\mathcal{B}}\xspace}
\newcommand{\bel}{\ensuremath{b}\xspace}
\newcommand{\belp}{\ensuremath{b'}\xspace}
\newcommand{\belInit}{\ensuremath{b_0}\xspace}
\newcommand{\belS}[1]{\ensuremath{b(#1)}\xspace}

\newcommand{\stSpace}{\ensuremath{\mathcal{S}}\xspace}
\newcommand{\stSpacep}{\ensuremath{\mathcal{S}'}\xspace}
\newcommand{\st}{\ensuremath{s}\xspace}
\newcommand{\stp}{\ensuremath{s'}\xspace}

\newcommand{\actSpace}{\ensuremath{\mathcal{A}}\xspace}
\newcommand{\act}{\ensuremath{a}\xspace}

\newcommand{\obsSpace}{\ensuremath{\mathcal{O}}\xspace}
\newcommand{\obs}{\ensuremath{o}\xspace}

\newcommand{\transF}{\ensuremath{T}\xspace}
\newcommand{\obsF}{\ensuremath{Z}\xspace}
\newcommand{\transFComp}{\ensuremath{T(\st, \act, \stp)}\xspace}

\newcommand{\rewFunc}{\ensuremath{R}\xspace}
\newcommand{\rewFuncComp}[2]{\ensuremath{R(#1, #2)}\xspace}

\newcommand{\hist}{\ensuremath{h}\xspace}

\newcommand{\pol}{\ensuremath{\pi}\xspace}
\newcommand{\optPol}{\ensuremath{\pi^*}\xspace}
\newcommand{\polSpace}{\ensuremath{\Pi}\xspace}

\newcommand{\polParam}{\ensuremath{\theta}\xspace}
\newcommand{\vecPolParam}{\ensuremath{{\bm\theta}}\xspace}
\newcommand{\paramSpace}{\ensuremath{\Theta}\xspace}
\newcommand{\vecMu}{\ensuremath{{\bm\mu}}\xspace}
\newcommand{\vecStd}{\ensuremath{{\bm\sigma}}\xspace}

\newcommand{\policyTree}{\ensuremath{T_{\policy}}\xspace}
\newcommand{\belTree}{\ensuremath{\mathcal{T}}\xspace}
\newcommand{\VT}{\ensuremath{\mathcal{H}}\xspace}
\newcommand{\episode}{\ensuremath{e}\xspace}

\newcommand{\cell}{\ensuremath{P}\xspace}

\newcommand{\ccite}[1]{~\citep{#1}}
\newcommand{\chref}[1]{Chapter~\ref{#1}}
\newcommand{\sref}{Section~\ref}
\newcommand{\appref}[1]{Appendix~\ref{#1}}
\newcommand{\aref}[1]{Algorithm~\ref{#1}}
\newcommand{\eref}[1]{eq.(\ref{#1})}
\newcommand{\fref}[1]{Figure~\ref{#1}}
\newcommand{\ffref}[2]{Figure~\ref{#1} and~\ref{#2}}
\newcommand{\ttref}[2]{Table~\ref{#1} and~\ref{#2}}
\newcommand{\tref}[1]{Table~\ref{#1}}
\newcommand{\lref}[1]{Lemma~\ref{#1}}
\newcommand{\dref}[1]{Definition~\ref{#1}}
\newcommand{\thref}[1]{Theorem~\ref{#1}}
\newcommand{\pref}[1]{Proposition~\ref{#1}}
\newcommand{\myfootref}[1]{$^{~\ref{#1}}$}
\newcommand{\subfigcap}[1]{\textit{#1}}
\newcommand{\subfig}[1]{\textit{#1}}
\newcommand{\proc}[1]{\mbox{\textsc{#1}}}

\newcommand{\solver}{Lazy Cross-Entropy Search Over Policy Trees\xspace}
\newcommand{\solverAbbr}{LCEOPT\xspace}

\newcommand{\ie}{i.e.\xspace}
\newcommand{\eg}{e.g.\xspace}
\newcommand{\expect}{\ensuremath{\mathbb{E}}}
\newcommand{\del}[1]{\mathrm{d}{#1}}
\newcommand{\indicator}[1]{\ensuremath{\mathbf{1}_{\{#1\}}}\xspace}
\newcommand{\comm}[2]{{\color{blue}\xspace#1}{\color{red}\xspace[#2]}}
\newcommand{\pomdpTuple}{\ensuremath{\mathcal{P}}\xspace}
\newcommand{\tree}{\ensuremath{\mathcal{T}}\xspace}
\newcommand{\node}{\ensuremath{\nu}\xspace}
\newcommand{\normal}{\ensuremath{\mathcal{N}}\xspace}
\newcommand{\vecComp}[2]{\ensuremath{#1_{(#2)}}\xspace}

\newcommand{\commNew}[1]{{\color{red} #1}}
\newcommand{\reuseTree}{RT}
\newcommand{\bellmanBackup}{BB}

\newcommand{\Naive}{Basic\xspace}
\newcommand{\naive}{basic\xspace}


\newcommand{\hpm}{\ensuremath{\hspace{-1em}\pm\hspace{-1em}}}

\maketitle

\begin{abstract}
The Partially Observable Markov Decision Process (POMDP) provides a principled framework for decision making in stochastic partially observable environments. However, computing good solutions for problems with continuous action spaces remains challenging. To ease this challenge, we propose a simple online POMDP solver, called \solver (\solverAbbr). At each planning step, our method uses a novel \emph{lazy} Cross-Entropy method to search the space of policy trees, which provide a simple policy representation. Specifically, we maintain a distribution on promising finite-horizon policy trees. The distribution is iteratively updated by sampling policies, evaluating them via Monte Carlo simulation, and refitting them to the top-performing ones. Our method is lazy in the sense that it exploits the policy tree representation to avoid redundant computations in policy sampling, evaluation, and distribution update. This leads to computational savings of up to two orders of magnitude. Our \solverAbbr is surprisingly simple as compared to existing state-of-the-art methods, yet empirically outperforms them on several continuous-action POMDP problems, particularly for problems with higher-dimensional action spaces.
\end{abstract}

\section{Introduction}\label{sec:Intro}
Decision making in stochastic partially observable environments is an essential, yet challenging problem in many domains, such as robotics\ccite{Kur22:Partially}, natural resource management \citep{filar2019pomdps} and cyber security \citep{schwartz2020pomdp}. The Partially Observable Markov Decision Process (POMDP) provides a principled framework to solve such decision making problems, by lifting the planning problem from an agent's state space to its \emph{belief space} \ie, the space of all probability distributions over the state space. While solving POMDPs exactly is computationally intractable in general\ccite{papadimitriou1987complexity}, many efficient approximately-optimal sampling-based online POMDP solvers have been developed (reviewed in \citet{Kur22:Partially}), making them viable tools for many realistic decision making problems under uncertainty.

However, solving POMDPs with continuous and high-dimensional action spaces remains challenging. Current state-of-the-art online solvers for POMDPs with continuous action spaces \citep{seiler2015online,sunberg2018online,mern2021bayesian,lim2020voronoi,hoerger22:ADVT} typically use Monte Carlo Tree Search (MCTS) \citep{coulom2007efficient} to find a near-optimal action amongst a sampled representative subset of the action space, often relying on partitioning of the action space, diminishing their performance for high-dimensional action spaces. 

We propose a new simple online POMDP solver for continuous action spaces, \solver (\solverAbbr), that uses a stochastic optimization approach in the policy space by extending the Cross-Entropy method for optimization \citep{Rubinstein2004:CrossEntropy, Boer2005tutorial} to compute a near-optimal policy, while avoiding any form of action-space partitioning. \solverAbbr represents a policy as a \emph{policy tree}, a compact and interpretable representation that gives rise to simple policy parameterizations via finite-dimensional vectors. Following the standard procedure of the Cross-Entropy method, \solverAbbr maintains a parameterized distribution over the policy parameters that is incrementally updated by sampling sets of parameters from the distribution and evaluating their associated policies via Monte Carlo sampling. The distribution is then updated towards the best-performing policies. This enables \solverAbbr to quickly focus its search on promising regions of the policy space.

\solverAbbr assumes independence of the marginal distributions over each component of the parameter vectors. This assumption allows us to derive a lazy parameter sampling, evaluation and distribution update method which only samples parts of a policy tree that are relevant for its evaluation. Our lazy approach reduces the cost of sampling policies by up to two orders of magnitude for problems with higher-dimensional action spaces, thereby significantly increasing the overall efficiency of \solverAbbr.

In contrast to many MCTS-based solvers, \solverAbbr avoids any form of partitioning of the action space, enabling it to scale much more effectively to problems with higher-dimensional action spaces. Despite its simplicity, \solverAbbr achieves remarkable results in various benchmark problems with continuous action spaces compared to current state-of-the-art methods, particularly for problems with higher-dimensional action spaces (up to 12-D). The source code of \solverAbbr is available at \url{https://github.com/hoergems/LCEOPT}.

\section{Related Work}\label{sec:pomdp_solvers}

Various efficient sampling-based online POMDP solvers have been developed for increasingly complex discrete and continuous POMDPs in the last two decades. In contrast to offline methods \ccite{bai2014integrated,kurniawati2011motion,Kurniawati08sarsop:efficient,Pin03:Point,Smi05:Point} that compute a policy offline before deployment, online solvers (e.g., \ccite{kurniawati2016online,silver2010monte,somani2013despot}) aim to further scale to larger and more complex problems by interleaving planning and execution, and focus on computing an optimal action for only the current belief during planning. For scalability purposes, \solverAbbr follows the online solving approach.

Some online solvers have been designed for continuous POMDPs, most of them being MCTS-based \citep{seiler2015online,sunberg2018online,mern2021bayesian,lim2020voronoi,hoerger22:ADVT} with some relying on partitioning the action space \citep{lim2020voronoi,hoerger22:ADVT}. 
These solvers do not scale well to problems with higher-dimensional action spaces though.

The Cross-Entropy method has been used in several algorithms for solving POMDPs and MDPs (the fully observable variant of POMDPs). Several of them consider discrete action spaces \citep{mannor2003cross, oliehoek2008cross, Wang2018}, while we consider POMDPs with continuous action spaces. 
\citet{omidshafiei2016graph} consider continuous actions spaces, but the optimization is carried out over a finite policy space. \citet{hafner2019learning} presents a Cross-Entropy based POMDP solver within a deep planning framework that optimizes over 
open-loop policies, while our method optimizes over closed-loop policies.


Additionally, some solvers\ccite{agha2011firm,sun2015high,van2011lqg,van2012motion} restrict beliefs to be Gaussian and use Linear-Quadratic-Gaussian (LQG) control\ccite{Lindquist73} to compute the best action. This strategy generally performs well in high-dimensional action spaces. However, they tend to perform poorly in problems with large uncertainties or non-Gaussian beliefs\ccite{hoerger2020linearization}. In contrast, our method puts no restriction on the class of beliefs, while simultaneously retaining efficiency in higher-dimensional action spaces.

\section{Preliminaries}\label{sec:background}
\paragraph{Partially Observable Markov Decision Process (POMDP)} \label{ssec:pomdp}

A POMDP provides a general mathematical framework for sequential decision making under uncertainty.
Formally, a POMDP is an 8-tuple $\langle \stSpace, \actSpace, \obsSpace, \transF, \obsF, \rewFunc, \bel_{0}, \gamma \rangle$. Initially, the agent is in a hidden state $s_{0} \in \stSpace$. This uncertainty is represented by an initial belief $\bel_0 \in \belSpace$, a probability distribution on the state space $\stSpace$, where \belSpace is the set of all possible beliefs. At each step $t \ge 0$, the agent executes an action $\act_{t} \in \actSpace$ according to some policy $\pol$. Due to stochastic effects of executing actions, it transitions from the current state $\st_t\in\stSpace$ to a next state $\st_{t+1} \in \stSpace$ according to the transition model $\transF(\st_{t}, \act_{t}, \st_{t+1}) = p(\st_{t+1} \vert \st_{t}, \act_{t})$. For discrete state spaces, $\transF(\st_{t}, \act_{t}, \st_{t+1})$ is often a probability mass function, whereas for continuous state spaces, it typically is a probability density function. The agent does not know the state $\st_{t+1}$ exactly, but perceives an observation $\obs_{t} \in\obsSpace$ from the environment according to the observation model $\obsF(\st_{t+1}, \act_{t}, \obs_{t}) = p(\obs_{t} \vert \st_{t+1}, \act_{t})$. 
In addition, the agent receives an immediate reward $r_{t} = R(\st_{t}, \act_{t}) \in \reals$. The agent's goal is to find a policy $\pol$ that maximizes the expected total discounted reward or the {\em policy value}
\begin{align}
    V_{\pol}(\bel_{0}) = \expect\left[\sum_{t=0}^{\infty}\gamma^t r_t \, \bigg\vert\, \bel_{0}, \pi\right],
\end{align}
where the discount factor $0 < \gamma < 1$ ensures that $V_{\pol}(\bel)$ is finite and well-defined.

The agent's decision space is the set $\Pi$ of policies, defined as mappings from beliefs to actions. The POMDP solution is then the optimal policy, denoted as \optPol and given by 
\begin{align}
\optPol = \argmax_{\pol \in \Pi} V_{\pol}(\bel),
\end{align}
for each belief $\bel\in\belSpace$.
A more elaborate explanation is available in \citet{kaelbling1998planning}.

\paragraph{Cross-Entropy Method for Optimization}\label{ssec:cross_entropy}

The Cross-Entropy (CE) Method \citep{Rubinstein2004:CrossEntropy, botev2013cross} is a gradient-free method for discrete and continuous optimization problems. Suppose $\mathcal{X}$ is an arbitrary solution space, and $f: \mathcal{X} \rightarrow \reals$ is an objective function that we aim to optimize, \ie, we aim to find $x^*\in\mathcal{X}$, such that $x^* = \argmax_{x\in\mathcal{X}}f(x)$. To do this, the CE-method iteratively constructs a sequence of sampling densities $d(\cdot; \eta_1), d(\cdot; \eta_2), \ldots, d(\cdot; \eta_T)$ over $\mathcal{X}$, with parameters $\eta_1, \ldots, \eta_T$ such that $d(\cdot; \eta_t)$ assigns more probability mass near $x^*$ as $t$ increases. 

In particular, suppose we start from an initial sampling density $d(\cdot; \eta_1)$. At iteration $1 \leq t \leq T$, the CE-method draws a sample of candidate solutions $X = \{x_i\}_{i=1}^N$ from $d(\cdot; \eta_t)$ and evaluates $f(x_i)$ for each $x_i\in X$. The sample objective values are then sorted in increasing order and are used to obtain the density parameter $\eta_{t+1}$ for the next iteration by solving the following maximum likelihood estimation problem:
\begin{equation}\label{eq:stochastic_program}
\eta_{t+1} = \argmax_{\eta} \frac{1}{N}\sum_{i=1}^N I_{\{f(x_i)\geq f_{(K)}\}}\mathrm{ln}(d(x_i, \eta)),
\end{equation}
where $f_{(K)}$ is the $K$-th largest sample objective value, with $0 < K \leq N$ being a user defined parameter. This process then repeats until the maximum number of iterations $T$ is reached, or some convergence criterion is met.

While solving \Cref{eq:stochastic_program} is generally intractable, analytic solutions exist for sampling densities of many commonly used distributions from the exponential family. For instance, in case $d$ is the density of a Gaussian distribution parameterized by $\eta = (\mu, \sigma^2)$, the solution of \cref{eq:stochastic_program} is $\hat{\eta} = (\hat{\mu}, \hat{\sigma}^2)$, with $\hat{\mu} = \frac{1}{\vert\mathcal{K}\vert}\sum_{x\in\mathcal{K}}x$ and $\hat{\sigma}^2 = \frac{1}{\vert\mathcal{K}\vert}\sum_{x\in\mathcal{K}}(x-\mu)^2$, where $\mathcal{K} = \{\bm x\in X\ \vert\ f(\bm x) \geq f_{(K)}\}$ are the top-$K$ performing samples, called \emph{elite samples}. That is, the updated distribution is a Gaussian distribution that is fitted to the elite samples. Similarly, if $\mathcal{X}$ is a multidimensional Euclidean space and $d$ is the density of a multivariate Gaussian distribution parameterized by $\eta = (\vecMu, \Sigma)$, the solution to \cref{eq:stochastic_program} is $\hat{\eta} = (\hat{\vecMu}, \hat{\Sigma})$, with $\hat{\vecMu} = \frac{1}{\vert\mathcal{K}\vert}\sum_{\bm x\in\mathcal{K}}\bm x$ and $\hat{\Sigma} = \frac{1}{\vert\mathcal{K}\vert}\sum_{\bm x\in\mathcal{K}}(\bm x-\vecMu)(\bm x -\vecMu)^\top$.


In practice, to avoid premature convergence towards a local optimum, $\eta$ is often updated according to a smoothed updating rule, \ie,
\begin{equation}
\hat{\eta} = (1-\alpha)\eta + \alpha\tilde{\eta}, 
\end{equation}
where $\tilde{\eta}$ is the solution to \cref{eq:stochastic_program}, and $0 < \alpha \leq 1$ is a smoothing parameter. More details on the CE-method for optimization can be found in \citet{Rubinstein2004:CrossEntropy, botev2013cross}.

\section{\solver}\label{sec:method}
We present the assumptions and an overview of our method \solver (\solverAbbr) in \Cref{Assumptions}
and \Cref{ssec:solver_overview} respectively, and then present the details in the following subsections.
We first describe our policy class and its parameterization in \Cref{ssec:policy_parameterization}, 
then describe how policy sampling, evaluation and distribution update are carried out in \Cref{ssec:sampling}.
Specifically, in \Cref{ssec:sampling_naive} we start with describing a basic method that highlights the conceptual framework of our approach but is computationally inefficient. We then describe a lazy method that is much more efficient and is actually used in our \solverAbbr algorithm. 
The key steps of \solverAbbr are shown in \Cref{alg:ceSolver}, with detailed pseudo codes provided in \Cref{asec:pseudo-codes}.

\subsection{Assumptions}\label{Assumptions}
We assume that the POMDP 
$\pomdpTuple = \langle \stSpace, \actSpace, \obsSpace, \transF, \obsF, \rewFunc, \bel_{0}, \gamma \rangle$
to be solved has a (bounded or unbounded) $D$-dimensional continuous action space $\actSpace$,
a discrete observation space $\obsSpace$, and an arbitrary (discrete or continuous or mixed) 
state space $\stSpace$.
Similarly to many existing online POMDP solvers, we assume access to a stochastic \emph{generative model} $G: \stSpace\times\actSpace \rightarrow \stSpace\times\obsSpace\times\mathbb{R}$ that simulates the transition, observation and reward models,
instead of requiring an explicit representation for them.
That is, for a given state $\st\in\stSpace$ and action $\act\in\actSpace$, the model $G$ generates a next state $\stp\in\stSpace$, observation $\obs\in\obsSpace$ and reward $r\in\mathbb{R}$ according to $(\stp, \obs, r) = G(\st,\act)$, where $(\stp, \obs)$ is distributed according to $p(\stp, \obs \,\vert\, \st, \act) = \transF(\st, \act, \stp)\obsF(\stp, \act, \obs)$, and $r = R(\st, \act)$.

\begin{algorithm}[htb]
\caption{\textproc{\solverAbbr}}\label{alg:ceSolver}
\begin{algorithmic}[1]
\While{problem not terminated}
    \State Initialize current policy distribution $d_{0}$ and set $k = 0$
    \While{planning budget not exceeded}\label{lst:while_loop}
        \State \emph{Lazily} sample and evaluate candidate policies $\bm{\theta}_{1:N}$ from current policy distribution $d_{k}$  
        \State $\mathcal{K} \gets $\ Set of top-$K$ performing parameter vectors\label{lst:k}
        \State Compute the new policy distribution $d_{k+1}$ from $\mathcal{K}$ and $d_{k}$
        \State $k \gets k+1$        
    \EndWhile 
    \State $\vecMu \gets$ the mean of the current policy distribution
    \State Execute $\act^* = \pol_{\vecMu}(\bel)$ and update belief  
\EndWhile
\end{algorithmic}
\end{algorithm}

\subsection{Overview of \solverAbbr}\label{ssec:solver_overview} 
\solverAbbr is an anytime online POMDP solver based on a \emph{lazy} CE-method that 
can handle incomplete data.
\Cref{alg:ceSolver} shows the key steps of \solverAbbr.
Basically, at each time step (lines 2 to 8), \solverAbbr estimates the optimal parameter $\vecPolParam^{*}
= \argmax_{\vecPolParam} V_{\pi_{\vecPolParam}}(\bel)$ of a parametric policy
$\pi_{\vecPolParam}$. Details regarding the policy parameterization are discussed in \Cref{ssec:policy_parameterization}
To this end, \solverAbbr maintains and updates a distribution $d$ on the parameter space $\Theta$. 
Here we chose the distribution to be a multivariate Gaussian distribution $\normal(\vecMu, \mathrm{diag}(\vecStd^2))$, parameterized by a mean vector $\vecMu$ and a vector of variances $\vecStd^2$. The notation $\mathrm{diag}(\vecStd^2)$ denotes a diagonal covariance matrix whose main diagonal is $\vecStd^2$. While other distributions could be chosen, this particular choice enables us to derive efficient parameter sampling, evaluation and distribution update approaches, as we will discuss in \Cref{ssec:sampling_lazy}. The distribution is iteratively updated using the Cross-Entropy method as described below.

\solverAbbr starts from an initial distribution $d_0 = \normal(\vecMu_{\mathrm{init}}, \mathrm{diag}(\vecStd_{\mathrm{init}}^2))$ over \paramSpace. At each planning step $k$ (lines 4 to 7), \solverAbbr samples $N>0$ policy parameters $\vecPolParam_{1:N}$ from the distribution $d_k$ and, for each sampled $\vecPolParam_i$, computes a policy value estimate $\widehat{V}_{\vecPolParam_i}(\bel) \approx V_{\pol_{\vecPolParam_i}}(\bel)$ for the current belief \bel as the average return over multiple randomly sampled simulations. We then compute the updated distribution $d_{k+1}$ by computing the distribution parameters $\vecMu$ and $\vecStd^2$ as the mean an variance of the $K>0$ policy parameters with highest estimated policy values. 
This process repeats from the updated distribution parameterized by $\vecMu$ and $\vecStd^2$ until the planning budget for the current step has been exceeded. \Cref{ssec:sampling_naive} describes a basic method to sample and evaluate policy parameters and update the distribution parameters, which serves as both a baseline and a precursor to our more efficient and novel lazy method, described in the same section. In the implementation of \solverAbbr, we use the lazy method.

After planning ends at each time step, the action for the agent to execute is then chosen to be $\act^* = \pol_\vecMu(\bel)$ (lines 9 to 10). After executing the action and perceiving an observation $\obs\in\obsSpace$, we update the belief to $\bel' = \tau(\bel, \act^*, \obs)$, where $\tau$ is the Bayesian belief update function. Our implementation uses a Sequential Importance Resampling (SIR) particle filter\ccite{arulampalam2002tutorial} to update the belief. This process is then repeated from the updated belief until some terminal condition is satisfied (line 1).

\subsection{Policy Parameterization}\label{ssec:policy_parameterization}
To facilitate a simple policy parameterization and derive an efficient method to evaluate a policy, \solverAbbr represents each policy \pol as a \emph{policy tree} $\tree_\pol$. From now on, we drop the subscript in $\tree_\pol$ and implicitly assume that \tree represents policy \pol. A policy tree is a tree whose nodes represent actions and whose edges represent observations. It describes a decision plan, such that the agent starts by executing the action associated with the root node of \tree. After perceiving an observation from the environment, the agent follows the edge representing the perceived observation and the process repeats from the child node of the followed observation edge. For infinite-horizon problems, the depth of a policy tree may be infinite, which makes defining a suitable policy parameterization difficult. 
Thus, in this paper, we restrict the space of policies to be the space $\polSpace_M\subset\polSpace$ of all policies represented by policy trees of depth $M$, where $M>0$ is a user defined parameter.

\begin{figure}[htb]
\centering
\small
\includegraphics[width=0.4\textwidth]{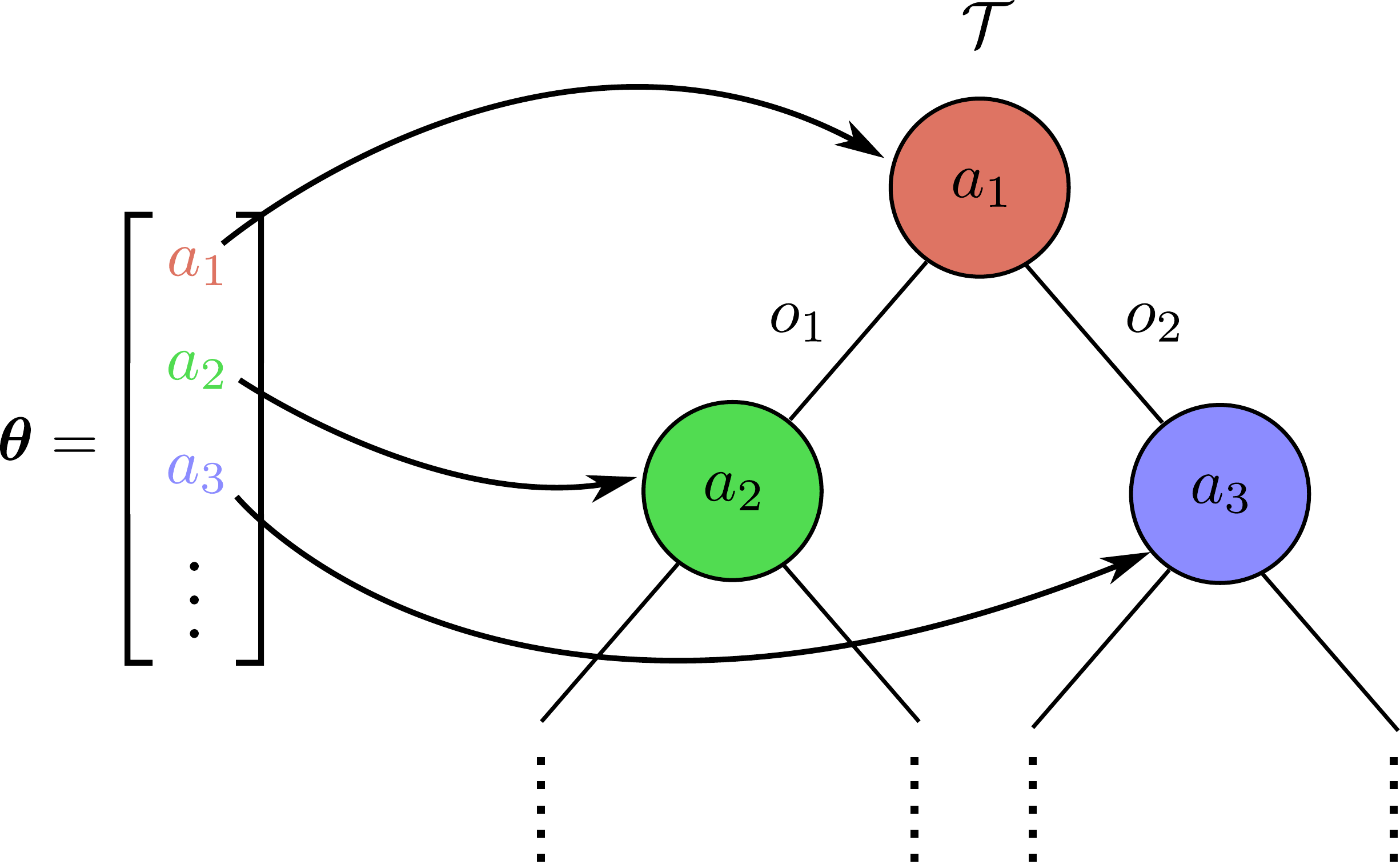}
\caption{Illustration of the relationship between the parameter vector \vecPolParam (left) and the policy tree $\tree$ (right), representing policy \pol. The components of \vecPolParam are actions that are associated with the action nodes in $\tree$.}
\label{f:PolicyTree}
\end{figure}

Policy trees provide a compact and interpretable representation of policies that give rise to a simple parameterization: A policy tree can be parameterized by a $(D\left\vert\tree \right\vert)$-dimensional vector \vecPolParam, such that each component \vecComp{\polParam}{\node} of \vecPolParam corresponds to the action associated with a particular node $\node\in\tree$ in the tree. Here, $D$ denotes the dimensionality of the action space, while $\vert\tree\vert$ denotes the number of nodes in \tree, which is equal to $(1-\vert\obsSpace\vert^{M+1})/(1-\vert\obsSpace\vert)$, where $\vert\obsSpace\vert$ is the cardinality of the observation space. \Cref{f:PolicyTree} illustrates the relationship between the parameter vector \vecPolParam and the policy tree \tree. 


\subsection{Policy Sampling, Evaluation and Distribution Update} \label{ssec:sampling}
We first describe a basic method for sampling and evaluating policy parameters and update the distribution parameters, followed by a discussion on our lazy method.

\subsubsection{The \Naive Method}\label{ssec:sampling_naive}


Our \naive method is a simple approach that highlights the conceptual framework of our \solverAbbr algorithm
and serves as a precursor to the more efficient lazy method described in the next subsection.

During planning, given the current policy parameter distribution $\normal(\vecMu, \mathrm{diag}(\vecStd^2))$, we sampe policy parameters $\vecPolParam_{1:N}$ directly from the distribution.
For each sampled policy $\pi_{\vecPolParam}$, we compute an estimate of the policy value $V_{\pol_{\vecPolParam}}(\bel)$ using Monte Carlo sampling. 
Specifically, starting from the current belief, we sample $L>0$ state trajectories by simulating $\pi_{\vecPolParam}$ and use the average of the accumulated discounted rewards of the trajectories as an approximation to $V_{\pol_{\vecPolParam}}(\bel)$. Once a trajectory reaches a leaf node in the policy tree $\tree_{\pi_{\vecPolParam}}$, we use the final sampled state to compute a heuristic estimate of the optimal value $V^*(\bel')$, where $\bel'\in\belSpace$ is the belief, conditioned on the action and observation sequences of the sampled trajectory. This heuristic estimate can be obtained in various ways, for instance by simulating a rollout policy, similarly to MCTS-based online POMDP solvers\ccite{silver2010monte,seiler2015online,sunberg2018online,hoerger22:ADVT}, or by hand-crafting estimates that exploit domain knowledge. Naturally, as we will see in \Cref{sec:experiments}, a good estimate of $V^*(\bel')$ often allows us to plan with a relatively short planning horizon while still achieving good policy performance.

After sampling and evaluating $N$ policy parameters as above, we update the parameters of the distribution over \paramSpace, based on the $K$ best performing parameters as shown in \Cref{alg:update_distr_basic} in Appendix A. In particular, we first compute new distribution parameters $\tilde{\vecMu}$ and $\tilde{\vecStd}^2$ as the mean and variance of the elite parameter vectors. The final distribution parameters \vecMu and $\vecStd^2$ are then computed according to $\vecMu\gets(1-\alpha)\vecMu + \alpha\tilde{\vecMu}$ and $\vecStd^2\gets(1-\alpha)\vecStd^2 + \alpha\tilde{\vecStd}^2$ respectively (\cref{lst3:mu_new,lst3:sig_new}), where $0\leq \alpha \leq 1$ is a user-defined smoothing parameter. As discussed in \Cref{ssec:cross_entropy}, this smooth update rule helps to avoid premature convergence of the distribution towards sub-optimal regions in \paramSpace.

\subsubsection{The Lazy Method}\label{ssec:sampling_lazy}

A key limitation of the basic method discussed in the previous section is the high computational cost of sampling full parameter vectors, particularly for high-dimensional action spaces. 
This is because, while sampling parameter vectors from the diagonal multivariate Gaussian distribution is simple, the large number of sampled 
parameter vectors can make it a computational bottleneck in the algorithm.
For instance, in a problem with a 12 dimensional action space, the basic method takes tens of seconds for just one iteration of the CE-method,
which makes it impractical in many applications (see \Cref{ssec:effects_lazy}).

Our lazy method provides a much more efficient way to implement the CE-method by using a key observation:
large portions of a parameterized policy tree are often irrelevant when estimating its policy value, since the sampled trajectories used for the evaluation may not reach them.
Based on this observation, we employ a lazy sampling method that only samples visited components of a parameter vector, where a component is visited if a sampled trajectory reaches its associated node in the policy tree.

To sample $N>0$ new parameter vectors $\vecPolParam_{1:N}$, we start by constructing $N$ vectors of size $D(1-\vert\obsSpace\vert^{M+1})/(1-\vert\obsSpace\vert)$ whose elements are set to the null symbol $\varnothing$.
For each policy $\pi_{\vecPolParam}$, once a sampled trajectory reaches a node $\node$ in the policy tree associated to $\pi_{\vecPolParam}$, we check whether the sub-vector $\vecComp{\polParam}{\node}$, \ie, the action associated with node \node, has already been sampled. If this is not the case, we sample a new action from the distribution $\normal(\vecComp{\mu}{\node}, \vecComp{\sigma}{\node}^2 I)$, which is the marginal of $\normal(\vecMu, \mathrm{diag}(\vecStd^2))$ corresponding to $\vecComp{\polParam}{\node}$, and assign the sampled action to $\vecComp{\polParam}{\node}$. Note that we sample $\vecComp{\polParam}{\node}$ only once, and keep it fixed for the remainder of the trajectory sampling process. 

The above sampling method results in a set of sampled parameter vectors for which some of the components are $\varnothing$, if their corresponding nodes have never been visited. Consequently, we have to slightly modify the distribution update step of the \naive method which computes new distribution parameters $\vecMu$ and $\vecStd^2$ based on the elite parameter vectors $\mathcal{K}$. In particular, we update the marginal distributions corresponding to each dimension of the parameter space independently, based on the entries of the elite parameter vectors in $\mathcal{K}$ that are not $\varnothing$. 
That is, for the parameter dimension $i$, we compute the marginal distribution parameters according to 
\begin{align}
    \tilde{\mu}_i &= \frac{1}{N_i}\sum_{\vecPolParam\in\mathcal{K}}\indicator{\polParam_i \neq \varnothing} \polParam_i, \label{eq:lazymu}\\
    \tilde{\sigma}_i^2 &= \frac{1}{N_i}\sum_{\vecPolParam\in\mathcal{K}}\indicator{\polParam_i \neq \varnothing}(\polParam_i - \tilde{\mu}_i)^2 \label{eq:lazysigma2}
\end{align} 
respectively\footnote{Note that $\theta_{i}$ denotes a single number, while \vecComp{\polParam}{\node} can possibly be a vector.}, where $\indicator{\cdot}$ denotes the indicator function, and $N_i = \sum_{\vecPolParam\in\mathcal{K}}\indicator{\polParam_i \neq \varnothing}$ is the number of parameter vector entries along the $i$-th dimension that are not $\varnothing$. 
Note that we only update the marginal distribution parameters at the $i$-th dimension if $N_i>0$. If $N_i = 0$, we simply set $\tilde{\mu}_i = \mu_i$ and $\tilde\sigma_i^2 = \sigma_i^2$. Similarly to the basic version of the distribution update, we compute the final marginal distribution parameters according to $\mu_i\gets(1-\alpha)\mu_i + \alpha\tilde{\mu}_i$ and $\sigma_i^2\gets(1-\alpha)\sigma_i^2 + \alpha\tilde{\sigma}_i^2$ respectively.



While the lazy algorithm is designed to speed up the \naive algorithm, they usually do not compute identical results, due to different sampling and distribution update strategies. However, the lazy update algorithm can still be derived from the standard cross-entropy framework described in \Cref{ssec:cross_entropy}, as shown by the following simple but interesting result.

\begin{thm} \label{thm:lazyce}
    Consider a set of (possibly partial) parameter vectors $\mathcal{K}$,
    a multivariate Gaussian distribution $\mathcal{N}(\vecMu, \diag(\vecStd^{2}))$, and the 
    the maximum likelihood estimation problem in the cross-entropy method
    \begin{equation}
        \tilde{\vecMu}, \tilde{\vecStd}^2 
        = 
        \argmax_{\vecMu, \vecStd^{2}}
        \sum_{\vecPolParam \in \mathcal{K}}
        \ln( \normal(\vecPolParam; \vecMu, \diag(\vecStd^{2}))),
    \end{equation}
    where $N(\vecPolParam; \vecMu, \diag(\vecStd^{2}))$  denotes the marginal probability density on the non-null dimensions.
    Then the solution of the above problem is given by \Cref{eq:lazymu,eq:lazysigma2}.
\end{thm}
\begin{proof}
    The proof follows easily from the following observation:
    \begin{align*}
        &\quad 
        \sum_{\vecPolParam \in \mathcal{K}} \ln( \normal(\vecPolParam; \vecMu, \diag(\vecStd)^{2}))\\
        &=
        \sum_{\vecPolParam \in \mathcal{K}} 
        \sum_{i: \polParam_{i} \neq \varnothing} \ln( \normal(\polParam_{i}; \mu_{i}, \sigma_{i}^{2}))\\
        &= 
        \sum_{i} \sum_{\vecPolParam \in \mathcal{K}: \polParam_{i} \neq \varnothing}
        \ln( \normal(\polParam_{i}; \mu_{i}, \sigma_{i}^{2})).
    \end{align*}
    Thus we can estimate the $\mu_{i}$ and $\sigma_{i}^{2}$ for each dimension using the 
    standard maximum likelihood estimates for the mean and variance of a univariate Gaussian to obtain
    the formulas in the theorem.
\end{proof}

While \Cref{thm:lazyce} and its proof are very simple, the result reveals an important insight that our lazy parameter
sampling method enables the CE-method to be applied to incomplete data, while the standard CE-method only considers
complete data.

Using the above lazy parameter sampling method can lead to significant computational savings, since only the components of the parameter vectors are sampled that are relevant for evaluating the corresponding policy tree. In our experiments, we investigate the amount of computational savings of our lazy sampling and evaluation method compared to the \naive version discussed in the previous section.

\section{Experiments and Results}\label{sec:experiments}
We tested \solverAbbr on 4 decision making problems under partial observability:
\begin{itemize}
\item \textbf{ContTag}: ContTag \citep{seiler2015online} is an extension of the popular benchmark problem Tag \citep{Pin03:Point} to continuous action spaces.
\item \textbf{Pushbox2D/3D}: Pushbox2D/3D \citep{seiler2015online} is a scalable motion planning problem, based on air hockey in which an agent must push a disk-shape opponent into a goal area in the environment, while avoiding to push it into a boundary region.
\item \textbf{Parking2D/3D}: Parking2D/3D \citep{hoerger22:ADVT} is a navigation problem in which a vehicle must park between in a corridor between two obstacles, while having imperfect information regarding its starting location.
\item \textbf{SensorPlacement-D}: SensorPlacement-D \citep{hoerger22:ADVT} is a scalable motion planning under uncertainty problem, where a manipulator with $D$ degrees of freedom (DOF) and $D$ revolute joints (with $D\in\left \{6, 8, 10, 12 \right \}$) operates inside a 3D environment with muddy water and must attach a sensor at a marine structure while being subject to control errors.
\end{itemize}
Details regarding the problem scenarios are presented in \Cref{sec:problem_scenarios}. \Cref{ssec:experimental_setup} details the experimental setup, while the results are discussed in \Cref{ssec:results}.


\subsection{Experimental Setup}\label{ssec:experimental_setup}
The purpose of our experiments is two-fold. The first one is to compare \solverAbbr with three state-of-the-art online POMDP solvers for continuous action spaces, POMCPOW\ccite{sunberg2018online}, VOMCPOW\ccite{lim2020voronoi} and ADVT\ccite{hoerger22:ADVT} on the above problem scenarios. To do this, we implemented \solverAbbr, the tree baseline solvers POMCPOW, VOMCPOW and ADVT, and the problem scenarios in C++ using the OPPT framework\ccite{hoerger2018software}. All evaluated solvers have parameters that need to be tuned for achieving good performance, including the depth of the lookahead trees for POMCPOW, VOMCPOW and ADVT, and the policy trees for \solverAbbr respectively. To approximately determine the best parameters for each solver in the problem scenarios, we used the CE-method to search the solver's parameter spaces. The parameters for each solver and their searched value ranges are detailed in \Cref{asec:params}. 
For each solver and problem scenario, we then used the best parameter point and ran $1,000$ simulation runs with a fixed planning time of $1$s (measured in CPU time) per planning step.

The second purpose is to understand the computational benefits of our proposed lazy parameter sampling, evaluation, and distribution update method compared to the \naive method, both described in \Cref{ssec:sampling_naive}. To investigate this, we implemented a variant of \solverAbbr that uses the \naive method and tested both variants of \solverAbbr on the ContTag and SensorPlacement-12 problems. For both algorithms and problems, we measure the average CPU time required to reach $50$ CE-iterations per planning step, for different sizes of the policy trees. Here, a CE-iteration refers to one iteration within the while-loop in \Cref{alg:ceSolver} (\cref{lst:while_loop}), \ie, sampling and evaluating a set of policy parameters and updating the distribution over policy parameters. To see whether there is a notable difference in the quality of the policies computed by both algorithms, we tested them on the ContTag and SensorPlacement-6 problems, where we used a fixed number of $50$ CE-iterations per planning step for both algorithms and problems. We then ran $2,000$ simulation runs for each algorithm and problem. For both algorithms, we used the same parameters that were used for comparing \solverAbbr with the state-of-the-art methods.


All simulations were run single-threaded on an AMD EPYC 7003 CPU with 4GB of memory. The next section discusses the results of our experiments.
	
\subsection{Results}\label{ssec:results}
\begin{table*}[ht]
\small
\centering
\caption{Average total discounted rewards and $95\%$ confidence intervals of all tested solvers for the ContTag, Pushbox, Parking and SensorPlacement problems. The average is taken over $1,000$ simulation runs per solver and problem, with a planning time of 1s per step. The best result for each problem scenario is highlighted in bold.}\label{t:results1}
\begin{adjustbox}{max width=\textwidth}
\begin{tabular}{l*{5}{>{\hspace{0.0em}}rcl}}
& \multicolumn{3}{c}{ContTag} & \multicolumn{3}{c}{Pushbox2D} & \multicolumn{3}{c}{Pushbox3D} & \multicolumn{3}{c}{Parking2D} & \multicolumn{3}{c}{Parking3D} \\ \hline
\solverAbbr (Ours) & $0.02$ & \hpm & $0.23$ & $\mathbf{399.7}$ & \hpm & $\mathbf{8.7}$ & $\mathbf{358.6}$ & \hpm & $\mathbf{12.3}$ & $\mathbf{53.4}$ & \hpm & $\mathbf{0.4}$ & $\mathbf{47.2}$ & \hpm & $\mathbf{0.6}$ \\
ADVT & $\mathbf{0.37}$ & \hpm & $\mathbf{0.18}$ & $356.9$ & \hpm & $9.9$ & $327.8$ & \hpm & $14.7$ & $43.1$ & \hpm & $2.1$ & $34.6$ & \hpm & $2.1$ \\
VOMCPOW & $-1.95$ & \hpm & $0.31$ & $323.5$ & \hpm & $12.8$ & $145.7$ & \hpm & $13.7$ & $1.3$ & \hpm & $1.9$ & $-11.7$ & \hpm & $1.3$ \\
POMCPOW & $-2.00$ & \hpm & $0.31$ & $96.7$ & \hpm & $15.4$ & $25.9$ & \hpm & $12.2$ & $-3.9$ & \hpm & $1.8$ & $-18.4$ & \hpm & $1.1$ \\ \\
& \multicolumn{3}{c}{\hspace{0em}SensorPlacement-6} & \multicolumn{3}{c}{\hspace{0em}SensorPlacement-8} & \multicolumn{3}{c}{\hspace{0em}SensorPlacement-10} & \multicolumn{3}{c}{\hspace{0em}SensorPlacement-12} \\ \hline
\solverAbbr (Ours) & $\mathbf{914.3}$ & \hpm & $\mathbf{2.6}$ & $\mathbf{885.5}$ & \hpm & $\mathbf{2.9}$ & $\mathbf{858.8}$ & \hpm & $\mathbf{4.2}$ & $\mathbf{832.1}$ & \hpm & $\mathbf{4.5}$ \\
ADVT & $859.2$ & \hpm & $12.2$ & $794.1$ & \hpm & $15.3$ & $631.4$ & \hpm & $23.9$ & $456.8$ & \hpm & $28.2$ \\
VOMCPOW & $754.4$ & \hpm & $12.8$ & $540.5$ & \hpm & $17.2$ & $276.8$ & \hpm & $17.8$ & $73.6$ & \hpm & $12.1$ \\
POMCPOW & $354.5$ & \hpm & $19.9$ & $124.2$ & \hpm & $15.3$ & $12.2$ & \hpm & $8.2$ & $-6.0$ & \hpm & $4.9$ \\
\end{tabular}
\end{adjustbox}
\end{table*}


\begin{table*}[htb]
\small
\centering
\caption{Comparison of the time efficiency of the \naive and the lazy policy sampling strategies on the ContTag and SensorPlacement-12 problems.
The table shows the average CPU time (in seconds) to reach 50 CE-iterations for different policy tree depths. The average is taken over 20 planning steps. Larger values indicate a larger parameter sampling cost. For the ContTag problem, we set the number of candidate policies to $N=493$ and the number of trajectories per parameter vector to $L=103$ for both algorithms. For the SensorPlacement-12 problem, we set $N=496$ and $L=11$.}\label{t:cpu_time}
\begin{tabular}{rl|ccccc}
& & $M=1$ & $M=2$ & $M=3$ & $M=4$ & $M=5$ \\ \hline
ContTag & Lazy & $0.43$ & $0.64$ & $0.90$ & $1.19$ & $1.39$ \\
& \Naive & $0.43$ & $0.64$ & $0.91$ & $1.23$ & $1.57$ \\
 \midrule
SensorPlacement-12 & Lazy & $1.09$ & $1.58$ & $2.2$ & $3.07$ & $5.13$ \\
& \Naive & $2.65$ & $12.61$ & $137.56$ & $900.97$ & $3928.45$
\end{tabular}
\end{table*}

\begin{table}[htb]
\centering
\caption{Average total discounted rewards and $95\%$ confidence intervals of \solverAbbr using the lazy and the \naive policy sampling strategies in the ContTag and SensorPlacement-12 problems. For both algorithms we use $50$ CE-iterations per planning step. The average is taken over 2,000 simulation runs for both algorithms and problems.}\label{t:res_vs_lazy}
\begin{tabular}{l*{2}{>{\hspace{0.0em}}rcl}}
 & \multicolumn{3}{c}{ContTag} & \multicolumn{3}{c}{SensorPlacement-6} \\ \hline
Lazy & $0.15$ & $\hpm$ & $0.16$ & $920.4$ & $\hpm$ & $1.8$ \\
\Naive & $-0.11$ & $\hpm$ & $0.16$ & $919.8$ & $\hpm$ & $1.8$
\end{tabular}
\end{table}


\subsubsection{Comparison with State-of-the-Art Methods}\label{ssec:comparison}
\Cref{t:results1} shows the average total discounted rewards of all tested solvers for the ContTag, Pushbox, Parking and SensorPlacement problems. \solverAbbr outperforms the baseline solvers in all problems, except for the ContTag problem, in which ADVT performs slightly better. 

Notably \solverAbbr significantly outperforms the baselines in the SensorPlacement problems. The results indicate that \solverAbbr scales substantially better to higher-dimensional action spaces. For instance, in the SensorPlacement-12 problem (which consists of a 12-dimensional continuous space), \solverAbbr achieves a better result than the best baseline, ADVT, in the 8-dimensional SensorPlacement-8 problem. A similar effect can be seen in the Parking problems, where the performance of \solverAbbr suffers only marginally compared to the baselines as the dimensionality of the action space increases. 

We conjecture that this is due to the action sampling strategies of the baselines. POMCPOW uses a simple uniform action sampling strategy, which does not take the value of already sampled actions into account. ADVT and VOMCPOW construct Voronoi cells in the action space at each sampled belief and bias their action sampling strategies towards cells with good performing representative actions. However, for higher-dimensional action spaces, these cells may be too large to quickly focus sampling towards near optimal regions in the action space. On the other hand, \solverAbbr is a partition-free method which uses distributions over the policy space. At each node in the policy trees, \solverAbbr maintains and updates a sampling distribution that quickly focuses its probability mass towards near-optimal regions in the action space. This property allows \solverAbbr to scale much more effectively to higher-dimensional action spaces, compared to the baselines.

Another interesting observation is that all solvers require only a relatively short planning horizon for most of the problem scenarios. For the ContTag, Pushbox and SensorPlacement problems, all solvers require an effective planning horizon of only two steps. The reason is that all solvers use heuristic estimates of $V^*(\bel)$ when the planning horizon is reached, that is, a leaf node in the policy trees of \solverAbbr and lookahead trees of POMCPOW, VOMCPOW and ADVT is reached. 
For all problem scenarios, we designed simple state-dependent heuristic estimates of $V^*(\bel)$ 
by removing partial observability and action noise from the problem. 
Details regarding the heuristic estimates are provided in \Cref{sec:heuristic}. Such simple heuristics are often useful in keeping the required effective planning horizon short. For the Parking problems, we require a slightly longer effective planning horizon of five steps to achieve good performance, because actions have potentially long-term consequences. For instance, if the vehicle decides to accelerate aggressively while navigating towards the goal, it may require multiple steps to decelerate in order to avoid crashing into an obstacle. Such long-term consequences of actions are often difficult to capture via simple state-dependent heuristics, leading to a longer effective planning horizon required to find good solutions.

\subsubsection{Comparison of the \naive Method and the Lazy Method}\label{ssec:effects_lazy}
\Cref{t:cpu_time} shows the average CPU time (measured in seconds) required for \solverAbbr with the lazy and the \naive parameter sampling method to reach $50$ CE-iterations per planning step for the ContTag and SensorPlacement-12 problems respectively, as we increase the policy tree depth $M$. It can be seen that for the ContTag problem, both the lazy and \naive methods perform similar for more shallower trees (up to $M=4$), while the lazy method performs slightly better for policy trees of depth $M=5$. However, the lazy method outperforms the \naive one significantly in the SensorPlacement-12 problem, even for shallow policy trees. 
The reason is that the dimensionality of the parameter space increases dramatically for deeper policy trees, due to the $12$-dimensional action space. As a consequence, sampling full parameter vectors becomes computationally too expensive. 
On the other hand, our lazy method only samples the components of the parameter vectors that are relevant to evaluate the associated policy. The number of relevant components of a parameter vector is typically much smaller than the dimensionality of the parameter space, which leads to significant computational savings when sampling parameter vectors lazily.

\Cref{t:res_vs_lazy} shows the average total discounted rewards for both the lazy and the \naive version of \solverAbbr in the ContTag and SensorPlacement-6 problems, where for both variants, the number of CE-iterations per planning step is set to $50$. It can be seen that despite the different policy distribution update behaviours as discussed in \Cref{ssec:sampling_lazy}, both algorithms perform similar in the ContTag and SensorPlacement problems. This indicates that the lazy algorithm is able to retain the good performance of the \naive one, while being much more efficient computationally.

\section{Conclusion}\label{s:conclusion}
Online POMDP solvers have seen tremendous progress in the last two decades in solving increasingly complex decision making under uncertainty problems. Despite this progress, solving continuous-action POMDPs remains a challenge. In this paper, we propose a simple online POMDP solver, called \solver (\solverAbbr) designed for POMDP problems with continuous state and action spaces. \solverAbbr uses a lazy version of the CE-method on the space of policy trees to find a near-optimal policy. Despite its simple structure, \solverAbbr shows a strong empirical performance against state-of-the-art methods on four benchmark problems, particularly on those with higher-dimensional action spaces. These results indicate that gradient-free optimization methods that do not rely on partitioning the search space are viable tools for solving continuous POMDPs. 

An interesting avenue for future work is to generalize our method to POMDPs with continuous observation spaces. This would allow us to consider an even larger class of POMDPs.
In addition, our lazy CE method can be used to handle problems where the objective function 
can be evaluated using just a subset of the parameters.
We can perform lazy sampling to sample only the relevant parameters, and perform distribution
update by maximizing the marginal likelihood of the partial parameter vectors.
This may be useful in other applications.

\paragraph{Acknowledgements}
This work is partially supported by the Australian Research Council (ARC) Discovery Project 200101049.

\bibliography{references}

\clearpage

\appendix
\section{Pseudo-Code of \solverAbbr}\label{asec:pseudo-codes}
Here we present the pseudo codes of \solverAbbr. \Cref{alg:ceSolver_detailed} shows the overview of \solverAbbr, while \Cref{alg:evaluate_policy} and \Cref{alg:update_distr_basic} show the \naive policy sampling and distribution update approaches. Finally, \Cref{alg:evaluate_policy_lazy} and \Cref{alg:update_distr_lazy} present our lazy policy sampling, evaluation and distribution update approach that is used in \solverAbbr.

\begin{algorithm}[htb]
\caption{\textproc{\solverAbbr}(Initial belief $\bel_0$, number of candidate policies per iteration $N>0$, number of elite samples $K>0$, number of trajectories $L>0$, initial distribution parameters $(\vecMu_{\mathrm{init}}, \vecStd_{\mathrm{init}}^2)$, smoothing parameter $0<\alpha\leq 1$)}\label{alg:ceSolver_detailed}
\begin{algorithmic}[1]
\State $\bel \gets \bel_0$
\State isTerminal $\gets$\ False
\While{isTerminal is False}
    \State $\vecMu \gets \vecMu_{\mathrm{init}},$\ \ \ $\vecStd^2 \gets \vecStd_{\mathrm{init}}^2$ \label{lst:init_distr}
    \While{planning budget not exceeded}\label{lst:while_loop_2}         
        \For{$i=1$ to $N$}\label{lst:inner_loop1}            
            \State // Sample and evaluate a candidate policy \label{lst:inner_loop2}
            \State $(\vecPolParam_i, \widehat{V}_i(\bel)) \gets\ $\textproc{SampleAndEvaluatePolicy}(\bel, $(\vecMu, \vecStd^2)$, $L$) \Comment{\Cref{alg:evaluate_policy_lazy}} \label{lst:inner_loop3}          
        \EndFor \label{lst:inner_loop4}
        \State // Sort evaluated parameters in increasing order according to their estimated values $\widehat{V}$
        \State $\mathcal{K} \gets $\ Set of top-$K$ performing parameter vectors\label{lst:k_2}
        \State // Update the distribution parameters
        \State $(\vecMu, \vecStd^2) \gets \textproc{UpdateDistribution}((\vecMu, \vecStd^2), \mathcal{K}, \alpha)$ \Comment{\Cref{alg:update_distr_lazy}}\label{lst:new_distr}              
    \EndWhile     
    \State $\act^* \gets \pol_{\vecMu}(\bel)$\label{lst:chosen_act}
    \State $(\obs$,\ isTerminal$) \gets$\ Execute $\act^*$    
    \State $\bel' \gets \tau(\bel, \act^*, \obs)$ \label{lst:belief_update}
    \State $\bel \gets \bel'$    
\EndWhile
\end{algorithmic}
\end{algorithm}

\begin{algorithm}[htb]
\caption{\textproc{SampleAndEvaluatePolicy\Naive}(Belief $\bel$, distribution parameters $(\vecMu, \vecStd^2)$, number of trajectories $L$)}\label{alg:evaluate_policy}
\begin{algorithmic}[1]
\State // Sample parameter vector \vecPolParam from a Multivariate Normal distribution parameterized by $(\vecMu, \vecStd^2)$.
\State $\vecPolParam \sim \normal(\vecMu, \mathrm{diag}(\vecStd^2))$\label{lst2:sample_param}
\State $\tree \gets\ $Construct policy tree parameterized by \vecPolParam
\State $M \gets$\ Depth of $\tree$
\State $\node \gets\ $Root node of $\tree_{\vecPolParam}$\label{lst2:set_root}
\For{$l=1$ to $L$} 
    \State isTerminal $\gets$ False     
    \State // Sample an initial state from \bel.
    \State $\st\sim\bel$\label{lst2:sample_state}
    \For{$m=1$ to $M$}
        \State $\act \gets \vecComp{\polParam}{\node}$\label{lst2:act_comp}
        \State // Sample a next state $\stp$, observation $\obs$ and immediate reward $r_m$ from the generative model $G$.
        \State $(\stp, \obs, r_m) \gets G(\st, \act)$\label{lst2:simulate_action}        
        \State $\node \gets\ $Child node of \node via observation edge \obs\label{lst2:next_node}
        \State $\st \gets \stp$                
        \If{\st is terminal}\label{lst2:terminal_r}
            \State isTerminal $\gets$ True
            \State \textbf{break}
        \EndIf        
    \EndFor    
    \If{isTerminal $=$ False}
        \State $r_{M+1} \gets $\ Heuristic(\st)\label{lst2:compute_heuristic}
    \Else
        \State $r_{M+1} \gets 0$
    \EndIf
    \State // Accumulated total discounted reward of trajectory $l$.
    \State $R_l \gets \sum_{m=1}^{M+1} \gamma^{m-1} r_m$\label{lst2:acum_rew}      
\EndFor
\State $V \gets \frac{1}{L}\sum_{l=1}^{L} R_l$\label{lst2:value_estimate}
\State \Return $(\vecPolParam, V)$
\end{algorithmic}
\end{algorithm}

\begin{algorithm}[htb]
\caption{\textproc{UpdateDistribution\Naive}(Distribution parameters $(\vecMu$, $\vecStd^2$), elite samples $\mathcal{K}$, smoothing parameter $\alpha$)}\label{alg:update_distr_basic}
\begin{algorithmic}[1]
\State $\tilde{\vecMu} \gets \frac{1}{\vert\mathcal{K}\vert}\sum_{\vecPolParam\in\mathcal{K}}\vecPolParam$;\ \ \ $\tilde{\vecStd}^2 \gets \frac{1}{\vert\mathcal{K}\vert}\sum_{\vecPolParam\in\mathcal{K}} (\vecPolParam - \tilde{\vecMu})^2$\label{lst3:new_param}
\State $\vecMu \gets (1-\alpha)\vecMu + \alpha \tilde{\vecMu}$\label{lst3:mu_new}
\State $\vecStd^2\gets (1-\alpha)\vecStd^2 + \alpha\tilde{\vecStd}^2$\label{lst3:sig_new}
\State \Return $(\vecMu, \vecStd^2)$
\end{algorithmic}
\end{algorithm}


\begin{algorithm}[!htb]
\caption{\textproc{SampleAndEvaluatePolicy}(Belief $\bel$, Distribution parameters $(\vecMu, \vecStd)$, Number of trajectories $L$)}\label{alg:evaluate_policy_lazy}
\begin{algorithmic}[1]
\State $\vecPolParam \gets $\ Construct empty parameter vector\label{lst4:empty_param}
\State $\tree_\vecPolParam \gets\ $Construct policy tree parameterized by $\vecPolParam$
\State $M \gets$\ Depth of $\tree_\vecPolParam$
\State $\node \gets\ $Root node of $\tree_{\vecPolParam}$
\For{$l=1$ to $L$} 
    \State isTerminal $\gets$ False
    \State $\st\sim\bel$
    \For{$m=1$ to $M$}
	\State $\act \gets \vecComp{\polParam}{\node}$
	\If{$\act = \varnothing$}\label{lst4:check_if}
	    \State // Sample \act from the marginal distribution at node $n$.
    	    \State $\act \sim $\ \normal$(\vecComp{\mu}{\node}, \mathrm{diag}(\vecComp{\sigma}{\node}^2))$\label{lst4:sample_new_1}
    	    \State $\vecComp{\polParam}{\node} \gets \act$\label{lst4:sample_new_2}
    	\EndIf
        \State $(\stp, \obs, r_m) \gets G(\st, \act)$        
        \State $\node \gets\ $Child node of \node via observation edge \obs
        \State $\st \gets \stp$                
        \If{\st is terminal}
            \State isTerminal $\gets$ True
            \State \textbf{break}
        \EndIf        
    \EndFor    
    \If{isTerminal $=$ False}
        \State $r_{M+1} \gets $\ Heuristic(\st)
    \Else
        \State $r_{M+1} \gets 0$
    \EndIf    
    \State $R_l \gets \sum_{m=1}^{M+1} \gamma^{m-1} r_m$      
\EndFor
\State $V \gets \frac{1}{L}\sum_{l=1}^{L} R_l$
\State \Return $(\vecPolParam, V)$
\end{algorithmic}
\end{algorithm}

\begin{algorithm}[htb]
\caption{\textproc{UpdateDistribution}(Distribution parameters $(\vecMu$, $\vecStd^2)$, elite samples $\mathcal{K}$, smoothing parameter $\alpha$)}\label{alg:update_distr_lazy}
\begin{algorithmic}[1]
\State // The term $D(1-\vert\obsSpace\vert^{M+1})/(1-\vert\obsSpace\vert)$ is the size of a parameter vector in $\mathcal{K}$.
\For{$i = 1$ to $D(1-\vert\obsSpace\vert^{M+1})/(1-\vert\obsSpace\vert)$}
    \State $N_i \gets \sum_{\vecPolParam\in\mathcal{K}}\indicator{\polParam_i \neq \emptyset}$\label{lst5:n_i}
    \If{$N_i > 0$}\label{lst5:check_if}
    \State $\tilde{\mu}_i \gets \frac{1}{N_i}\sum_{\vecPolParam\in\mathcal{K}} \indicator{\polParam_i \neq \emptyset} \polParam_i$\label{lst5:new_mu} 
    \State $\tilde{\sigma}_i^2 \gets \frac{1}{N_i}\sum_{\vecPolParam\in\mathcal{K}}\indicator{\polParam_i \neq \emptyset}(\polParam_i - \tilde{\mu}_i)^2$\label{lst5:new_sig}
    \State $\mu_i \gets (1-\alpha)\mu_i + \alpha\tilde{\mu}_i$\label{lst5:final_mu}
    \State $\sigma_{i}^2 \gets (1-\alpha)\sigma_i^2 + \alpha\tilde{\sigma}_i^2$\label{lst5:final_sig}
    \EndIf
\EndFor
\State \Return $(\vecMu, \vecStd^2)$
\end{algorithmic}
\end{algorithm} 


\section{Detailed Description of Problem Scenarios}\label{sec:problem_scenarios}
Here we provide a more detailed description of the problem scenarios used to evaluate \solverAbbr. The problem scenarios are illustrated in \Cref{f:problemScenarios}.
\begin{figure*}[htb]
\centering
\begin{tabular}{c@{\hskip5pt}c@{\hskip5pt}c@{\hskip5pt}c@{\hskip5pt}c}
\includegraphics[height=0.15\textwidth]{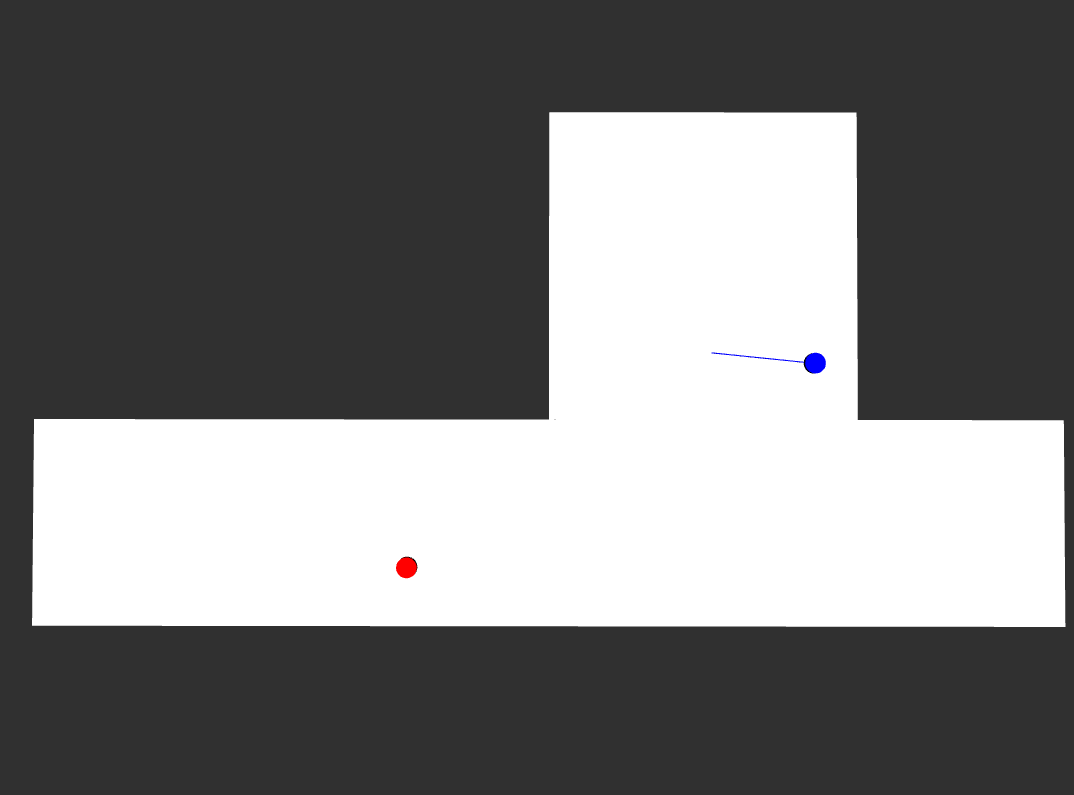} &
\includegraphics[height=0.15\textwidth]{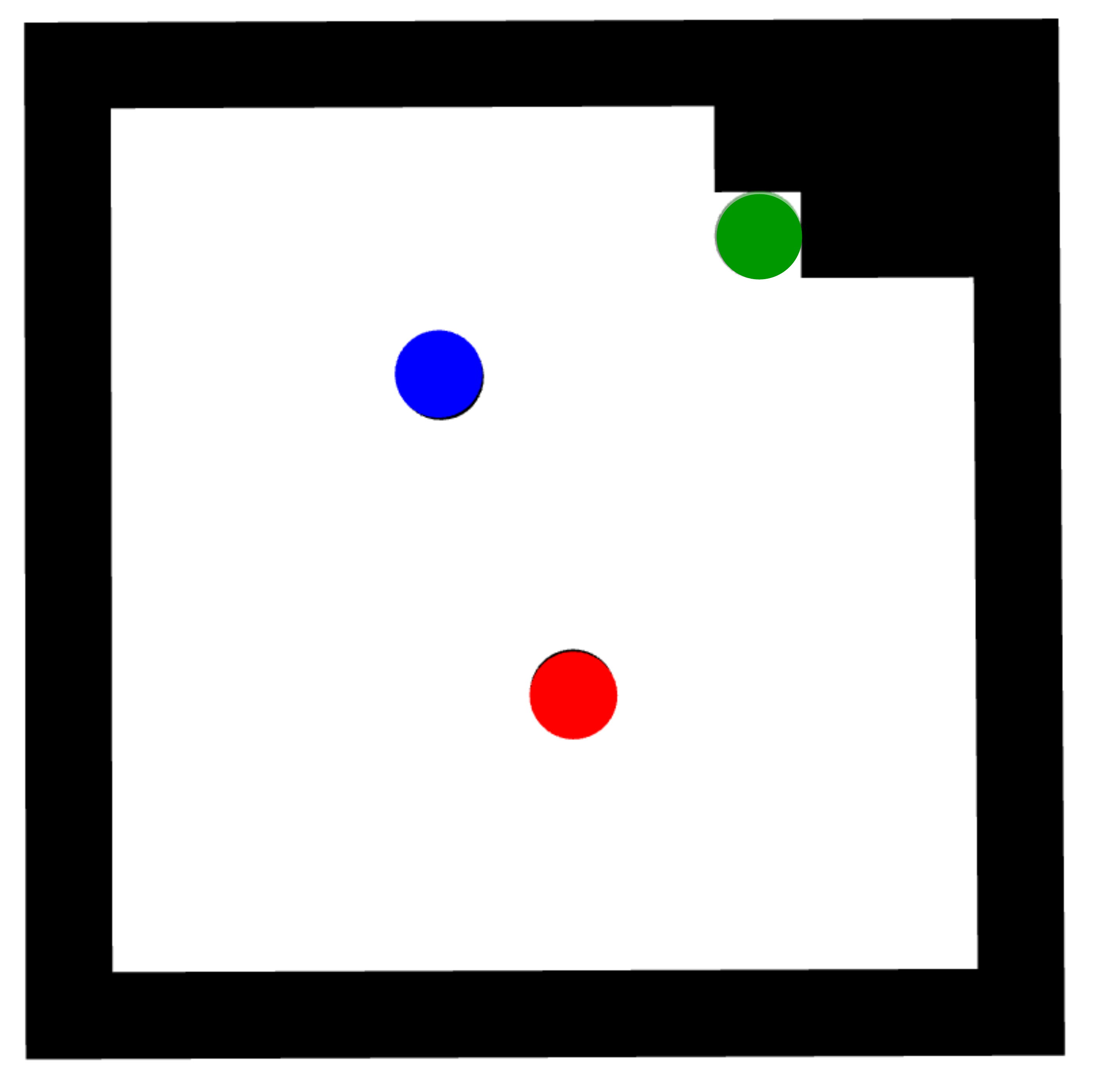} &
\includegraphics[height=0.15\textwidth]{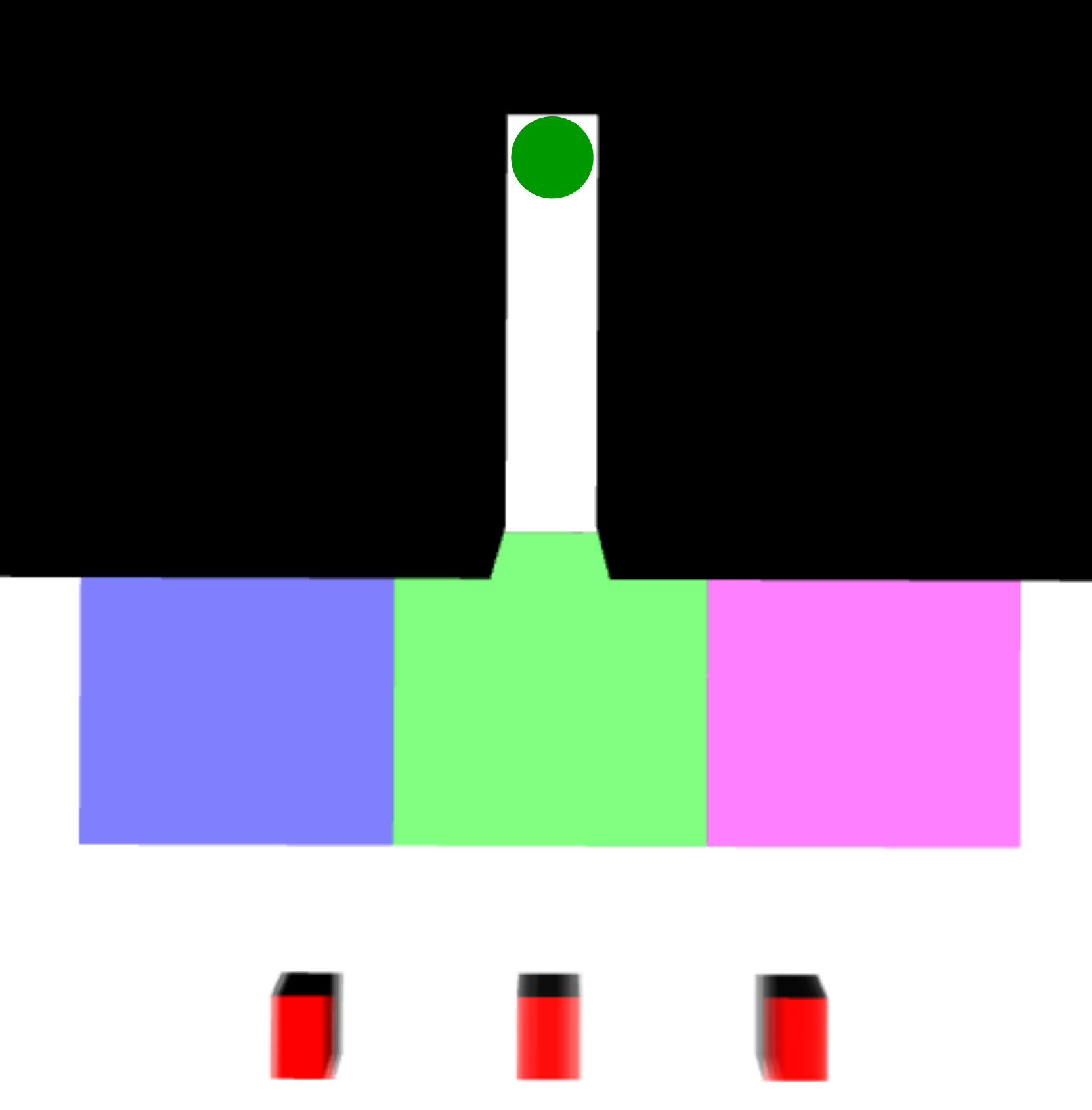} &
\includegraphics[height=0.15\textwidth]{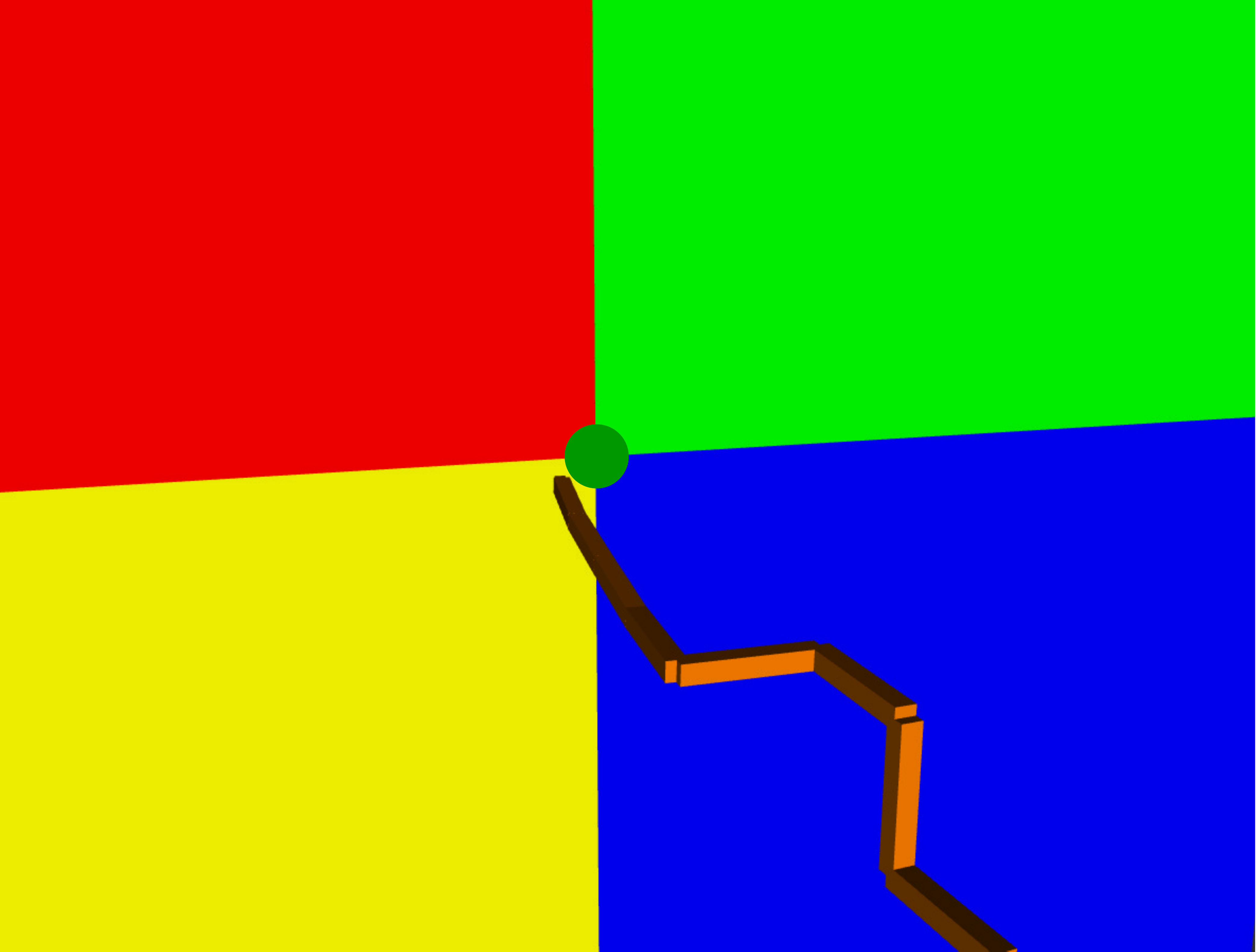} \\
(a) & (b) & (c) & (d)
\end{tabular}
\caption{Illustrations of (a) the ContTag, (b) the Pushbox2D, (c) the Parking2D and (d) the SensorPlacement-8 problems. The goal regions in the Pushbox2D, Parking2D and SensorPlacement-8 problems are marked as green circles. Images (b), (c) and (d) are taken from \citet{hoerger22:ADVT}.}
\label{f:problemScenarios}
\end{figure*}

\subsubsection{ContTag}
ContTag\ccite{seiler2015online} is a modified version of the popular POMDP benchmark problem Tag\ccite{Pin03:Point}. An agent operates in a 2D-environment (shown in \Cref{f:problemScenarios}(a)) where it has to tag an opponent, while the opponent is actively trying to avoid the agent. The state space is a five-dimensional continuous space consisting of the location $(x_r, y_r)$ and orientation $\phi_r$ (expressed in radians) of the agent and the location $(x_o, y_o)$ of the opponent. The action space is $\actSpace = [-\pi, \pi] \cup \{\textrm{TAG}\}$, where the first component is the set of all angular directions the agent can move towards, whereas the second component is an additional tag action. At each step, in case the agent executes a directional action, its orientation and position evolve deterministically according to $\phi_r' = \phi_r + \act$, $x_r' = x_r \mathrm{cos}(\phi_r')$ and $y_r' = y_r + \mathrm{sin}(\phi_r')$. Simultaneously, the opponent attempts to move away from the agent and its next location $(x_o', y_o')$ is computed according to $x_o' = x_o + \mathrm{cos}(\phi) + e_x$ and $y_o' = y_o + \mathrm{\sin}(\phi) + e_y$, where $\phi = \mathrm{atan2}(y_o - y_a, x_o - x_a)$ is the angle between the agent and the robot, and $e_x$ and $e_y$ are random motion errors drawn from a truncated Normal distribution $\normal(\mu, \sigma^2, l, u)$, which is the Normal distribution $\normal(\mu, \sigma^2)$ truncated to the interval $[l, u]$. For our experiments, we set $\mu = 0$, $\sigma = \frac{\pi}{8}$, $l = -\frac{\pi}{8}$ and $u = \frac{\pi}{8}$. In case the agent's or the opponent's next state would collide with the boundary region, their positions remain the same. If the agent executes the \textrm{TAG} action, its position and orientation remains unchanged as well. 

The initial positions of the agent and the opponent are drawn from a uniform distribution over the free space of the environment, while the initial orientation of the agent is set to $0$. While the agent knows its initial position and orientation, the position of the opponent is unknown. However, the agent has access to a noisy sensor with outputs $\{\textrm{DETECTED}, \textrm{NOT DETECTED}\}$ to detect the opponent. If the opponent is visible, \ie, the relative angle $\phi - \phi_r$ between the agent and the opponent is inside the interval $[-\frac{\pi}{2}, \frac{\pi}{2}]$, the sensor produces the output \textrm{DETECTED} with probability $p = 1-\frac{\phi - \phi_r}{\pi}$ and $\textrm{NOT DETECTED}$ with probability $1-p$. Otherwise, the sensor deterministically produces \textrm{NOT DETECTED}.

Upon activating the \textrm{TAG} action, the agent receives a reward of $10$ if its Euclidean distance to the agent is smaller than one unit length. Otherwise the agent receives a penalty of $-10$. Every other action incurs a small penalty of $-1$. The problem terminates if the opponent is successfully tagged, or a maximum of $90$ steps has been reached. The discount factor is $0.95$.

Note that the action space in this problem is a hybrid space, consisting of both continuous and discrete variables. For \solverAbbr, we embed the action space into the two-dimensional interval $[-\pi, \pi]\times [-1, 1]$ and define that the agent executes the \textrm{TAG} action, if the second component of the action is in the interval $[0, 1]$.
 
\subsubsection{Pushbox}
Pushbox \citep{seiler2015online}, as illustrated in \Cref{f:problemScenarios}(b), is a scalable motion planning problem, based on air hockey. A blue disk-shaped robot has to push a red disk-shaped puck into a green circle goal region, while avoiding any collisions with the black boundary region. If the puck is successfully pushed into the goal region, the robot receives a reward of $1,000$, but if either the robot or the puck collides with the boundary region, the robot receives a penalty of $-1,000$. Additionally, the robot incurs a penalty of $-10$ for every step taken. The robot can move around the environment by selecting a displacement vector. Upon colliding with the puck, the puck is pushed away, and its motion is affected by noise. The initial position of the robot is known, but the initial puck position is uncertain. However, the robot has access to a noisy bearing sensor to localize the puck. Additionally, the robot receives a binary observation from a contact sensor, indicating if a contact between the robot and the puck occurred. We consider two variants of the problem: \textbf{Pushbox2D} and \textbf{Pushbox3D}, which differ in the dimensionality of the state and action spaces. The former (illustrated in \Cref{f:problemScenarios}(b)) operates on a 2D plane, while the latter operates inside a 3D environment. More details on the Pushbox problem can be found in \citet{seiler2015online}.

\subsubsection{Parking}
The Parking problem, proposed in \citet{hoerger22:ADVT} and shown in \Cref{f:problemScenarios}(c), is a navigation problem in which a vehicle with deterministic dynamics operates in an environment populated by obstacles. The vehicle's goal is to safely reach a specified goal location while avoiding collisions with the obstacles. Reaching the goal earns a reward of 100, while collisions with obstacles incur a penalty of $-100$, and every step taken incurs a penalty of $-1$. The vehicle starts near one of three possible starting locations (red locations in \Cref{f:problemScenarios}(c)) with equal probability. There are three distinct areas in the environment with different types of terrain (colored areas in \Cref{f:problemScenarios}(c)), and the vehicle receives observations about the terrain type upon traversal. The observations are only correct 70\% of the time due to sensor noise. Here we consider two variants of the problem, \textbf{Parking2D} and \textbf{Parking3D}, with different state and action spaces. In Parking2D, the state consists of the vehicle's position, orientation, and velocity on a 2D plane, and the action space consists of the steering wheel angle and acceleration. In Parking3D, the vehicle operates in full 3D space, and the state and action spaces have additional components to model the vehicle's elevation and change in elevation, respectively. The problem is challenging due to multi-modal beliefs and the narrow passage to the goal, which makes good rewards scarce. Additional details can be found in \citet{hoerger22:ADVT}.

\subsubsection{SensorPlacement}
SensorPlacement, proposed in \citet{hoerger22:ADVT} and shown in \Cref{f:problemScenarios}(d), is a scalable motion planning under uncertainty problem, where a manipulator with $D$ degrees of freedom (DOF) and $D$ revolute joints operates inside a 3D environment with muddy water. The manipulator is situated in front of a marine structure, consisting of four walls (colored walls in \Cref{f:problemScenarios}(d)), and its task is to attach a sensor at a specific goal area located between the walls, which is reward by $1,000$, while avoiding collisions with the walls, which is penalized by $-500$. Additionally, the robot receives a penalty of $-1$ for every step. The state space consists of the joint angles for each joint, and the action space is a set of joint velocities. Initially, the robot is uncertain about its joint angle configuration and, due to underwater currents, the robot is subject to random control errors. To localize itself, the manipulator's end-effector is equipped with a touch sensor which provides noise-free information about the wall being touched. The problem has four variants, denoted as SensorPlacement-$D$, with $D$ ranging from 6 to 12, which differ in the number of revolute joints and the dimensionality of the action space. The discount factor is $\gamma=0.95$, and the robot must mount the sensor within 50 steps while avoiding collisions with the walls to succeed. Additional details can be found in\ccite{hoerger22:ADVT}.

\section{Heuristic Estimate of $V^*(\bel)$}\label{sec:heuristic}
In this section we provide a detailed description for each problem scenario on how the optimal value $V^*(\bel)$ is estimated at a leaf node, given the final state $\st\in\stSpace$ of a sampled trajectory.
In each case, we consider a simplified problem where partial observability and action noise are removed from the problem.
We then obtain a heuristic estimate for $V^{*}(\bel)$ by estimating the maximum total discounted reward achievable 
for the final state in the simplified problem.
This is done by computing a distance $\ell$ as an estimate for the number of steps needed to reach a certain configuration (e.g., for the agent to reach the opponent in ContTag), 
and treating $\ell$ as an integer for simplicity.
The estimate is crude and can be improved by obtaining better estimate for the number of steps to reach the desired configuration,
but we settled with the crude heuristic estimate as it performs well in our experiments.

\subsubsection{ContTag}
Suppose the variable $\ell$ denotes the Euclidean distance between the agent and the opponent for the final state $\st\in\stSpace$.
The heuristic estimate of $V^*(\bel)$ is computed via:

\begin{equation}\label{eq:heuristic_conttag}
\hat{V}^*(\bel) = \frac{1-\gamma^\ell}{1-\gamma}r_m   + \gamma^{\ell} r_t,
\end{equation}
where $r_m = -1$ is the step penalty, and $r_g = 10$ is the reward for succesfully tagging the opponent. The first term in \cref{eq:heuristic_conttag} estimates the total discounted cost of moving to the opponent, whereas the second term in \cref{eq:heuristic_conttag} estimates the discounted reward of tagging the opponent in the next step.

\subsubsection{Pushbox}
Similarly to ContTag, let $\ell$ be the Euclidean distance between the agent and the opponent for the final state $\st\in\stSpace$. The heuristic estimate of $V^*(\bel)$ is computed via:
\begin{equation}\label{eq:heuristic_pushbox}
\hat{V}^*(\bel) = \frac{1-\gamma^{\ell+1}}{1-\gamma}r_m + \gamma^{\ell}r_g,
\end{equation}
where $r_m = -1$ is the step penalty, and $r_g = 100$ the reward of pushing the opponent into the goal area. 
Here, the first term in \cref{eq:heuristic_pushbox} estimates the total discounted cost of reaching the opponent and pushing it into the goal area in the next step, whereas the second term in \cref{eq:heuristic_pushbox} estimates the discounted reward of pushing the opponent into the goal area in the next step after reaching the opponent.

\subsubsection{Parking and SensorPlacement problems}
We use the same heuristic estimate of $V^*(\bel)$ for the Parking and SensorPlacement problems, given the final state $\st\in\stSpace$ of a sampled trajectory. 
Suppose for the final state $\st$, the variable $\ell$ denotes the Euclidean distance between the vehicle and the goal in the Parking problem, and between the end-effector and the goal in the SensorPlacement problem respectively. We then compute a rough estimate of $V^*(\bel)$ via:
\begin{equation}\label{eq:heuristic_1}
\hat{V}^*(\bel) = \frac{1-\gamma^\ell}{1-\gamma}r_m   + \gamma^{\ell-1} r_t,
\end{equation}
where $r_m = -1$ is the step penalty in each problem, and $r_g$ is the reward for reaching the goal ($r_g=100$ in the Parking problem, and $r_g=1,000$ in the SensorPlacement problem). The first term in \cref{eq:heuristic_1} estimates the total discounted cost of reaching the goal, whereas the second term of \cref{eq:heuristic_1} estimates the discounted reward of reaching the goal in the same step.

\section{Solver Parameters}\label{asec:params}
\begin{table*}[htb]
\centering
\caption{Solver parameter and parameter ranges used when searching for the best parameters for all tested solvers in each problem scenario.}\label{t:param_ranges}
\begin{tabular}{lcccccc}
 & $N$ & $L$ & $K$ & $M$ & $\alpha$ & $\sigma_{\mathrm{init}}^2$ \\ \hline
\solverAbbr & $[10, 100]$ & $[1, 500]$ & $[1, 500]$ & $[1, 10]$ & $[0, 1]$ & $[0.01, 4.0]$\\\\
 & $C$ & $L$ & $C_r$ & \multicolumn{3}{c}{ } \\ \hline
ADVT & $[2, 500]$ & $[1, 500]$ & $[0.1, 100]$ &  \multicolumn{3}{c}{ } \\\\
 & $c$ & $k_a$ & $\alpha_a$ & $k_o$ & $\alpha_o$ & $\omega$ \\ \hline
VOMCPOW & $[2, 1]$ & $[1, 50]$ & $[0.001, 5]$ & $[1, 50]$ & $[0.001, 5]$ & $[0, 1]$  \\
POMCPOW & $[2, 1]$ & $[1, 50]$ & $[0.001, 5]$ & $[1, 50]$ & $[0.001, 5]$ & $-$
\end{tabular}
\end{table*}
\Cref{t:param_ranges} shows the parameter ranges used when searching for the best parameter of each solver. For all problem scenarios, we use the same parameter ranges. For \solverAbbr, the parameters $N$, $L$, $K$, $M$, $\alpha$ and $\vecStd_{\mathrm{init}}^2$ refer to the number of candidate policies per iteration, number of sampled trajectories per policy, number of elite samples, policy tree depth, smoothing factor and the variance of the initial distribution respectively. In all our experiments we set $\vecMu_{\mathrm{init}} = \mathbf{0}$ and $\vecStd_{\mathrm{init}}^2 = \sigma_{\mathrm{init}}^2\mathbf{1}$, where $\mathbf{0}$ and $\mathbf{1}$ are vectors of ones and zeroes respectively. Details regarding the parameters for ADVT can be found in \citet{hoerger22:ADVT}, while details regarding the parameters for VOMCPOW and POMCPOW can be found in \citet{lim2020voronoi}. To find the best set of parameters for each solver and problem scenario, we apply the CE-method for $100$ iterations, using a multivariate Gaussian distribution with diagonal covariance matrices (similarly to \solverAbbr). The best parameter is then chosen to be the mean of the resulting distribution over the parameter space.

\end{document}